\documentclass[conference]{IEEEtran}
\usepackage{amsmath,amssymb}

\usepackage[utf8]{inputenc}
\usepackage{graphicx} % more modern
\usepackage{epstopdf}
\usepackage{url}

\usepackage{tikz}
\usetikzlibrary{datavisualization}
\usepackage{lipsum}
\usepackage{booktabs}
\usepackage{pgfplots}
\usepackage{subfigure} 
\usepackage{mathtools}
\usepackage{color}
\usepackage{algorithm}
\usepackage{algpseudocode}

\newcommand{\paren}[1]{\left(#1\right)}
\newcommand{\Xspl}{\sS}

\newcommand{\distc}[1]{\rho\left(#1\right)}
\newcommand{\abs}[1]{\left|#1\right|}
\newcommand{\bx}{\boldsymbol{x}}

\newcommand{\sY}{\mathcal{Y}}

\newcommand{\sS}{\mathcal{S}}

\newcommand{\argmax}{\mathop{\mathrm{argmax}}}
\newcommand{\argmin}{\mathop{\mathrm{argmin}}}

\newcommand{\field}[1]{\mathbb{#1}}

\newcommand{\X}{\mathcal{X}}

\renewcommand{\Pr}{\field{P}}

\newcommand{\wh}{\widehat}
\newcommand{\ve}{\varepsilon}

\newcommand{\bool}{\{0,1\}}

\newtheorem{lemma}{Lemma}
\newtheorem{theorem}{Theorem}

\newcommand{\yhat}{\wh{y}}

\renewcommand{\hat}{\wh}

\newlength{\minipagewidth}
\newcommand{\bookbox}[1]{
	\par\medskip\noindent
	\framebox[\linewidth]{
		\begin{minipage}{\minipagewidth}
			{#1}
		\end{minipage} } \par\medskip }

	\newlength\figureheighta 
	\newlength\figurewidtha 
	\newlength\figureheightb 
	\newlength\figurewidthb 
	\newlength\figureheightc 
	\newlength\figurewidthc 
	\tikzset{every picture/.style={font issue={\fontsize{9}{10}}},font issue/.style={execute at begin picture={#1\selectfont}}}

\begin{document}
\title{The ABACOC Algorithm: a Novel Approach for Nonparametric Classification of Data Streams }

\author{\IEEEauthorblockN{Rocco De Rosa}
\IEEEauthorblockA{Dipartimento di Informatica\\
Universit\`a degli Studi di Milano, Italy
% rocco.derosa@unimi.it
}
\and
\IEEEauthorblockN{Francesco Orabona}
\IEEEauthorblockA{Yahoo Labs\\
	New York, NY, USA
%	francesco@orabona.com
}
\and
\IEEEauthorblockN{Nicolò Cesa-Bianchi}
\IEEEauthorblockA{Dipartimento di Informatica\\
	Universit\`a degli Studi di Milano,  Italy
%	nicolo.cesa-bianchi@unimi.it
} }

\maketitle

\begin{abstract}
Stream mining poses unique challenges to machine learning: predictive models are required to be scalable, incrementally trainable, must remain bounded in size (even when the data stream is arbitrarily long), and be nonparametric in order to achieve high accuracy even in complex and dynamic environments. Moreover, the learning system must be parameterless ---traditional tuning methods are problematic in streaming settings--- and avoid requiring prior knowledge of the number of distinct class labels occurring in the stream.
In this paper, we introduce a new algorithmic approach for nonparametric learning in data streams. Our approach addresses all above mentioned challenges by learning a model that covers the input space using simple local classifiers. The distribution of these classifiers dynamically adapts to the local (unknown) complexity of the classification problem, thus achieving a good balance between model complexity and predictive accuracy.
We design four variants of our approach of increasing adaptivity. By means of an extensive empirical evaluation against standard nonparametric baselines, we show state-of-the-art results in terms of accuracy versus model size. For the variant that imposes a strict bound on the model size, we show better performance against all other methods measured at the same model size value.
Our empirical analysis is complemented by a theoretical performance guarantee which does not rely on any stochastic assumption on the source generating the stream.\footnote{This paper is a longer version of the conference paper \cite{derosa2015abacoc}.}
\end{abstract}

\IEEEpeerreviewmaketitle

%NCB
%\input{intro.tex}

\section{Introduction}
As pointed out in various papers ---see, e.g., \cite{hulten2001mining,read2012scalable}--- stream mining poses unique challenges to machine learning: examples must be efficiently processed one at a time as they arrive from the stream, and an up-to-date predictive model must be available at all times. Incremental learning systems are well suited to address these requirements: the key difference between a traditional (batch) learning system and an incremental one is that the latter learns by performing small adjustments to the current predictor. Each adjustment uses only the information provided by the current example in the stream, allowing an efficient and timely update of the predictive model. This is unlike batch learning, where training typically involves a costly global optimization process involving multiple passes over the data.

Another important feature of stream mining is that the true structure of the problem is progressively revealed as more data are observed. In this context, nonparametric learning methods, such as decision trees or nearest neighbour (NN), are especially effective, as a nonparametric algorithm is not committed to any specific family of decision surfaces. For this reason, incremental algorithms for decision trees~\cite{domingos2000mining,rutkowski2012decision,matuszyk2013correcting,duda2014novel,de2015splitting} and nearest neighbour~\cite{zhang2011enabling} are extremely popular in stream mining applications.

Since in nonparametric methods the model size keeps growing to fit the stream with increasing accuracy, we seek a method able to improve predictions while growing the model as slowly as possible. However, as the model size cannot grow unbounded, we also introduce a variant of our approach that prevents the model size from going beyond a given limit. In the presence of concept drift~\cite{hulten2001mining,tsymbal2004problem}, bounding the model size may actually improve the overall predictive accuracy, provided the data point supporting the model are selected in the right way.

A further issue in stream mining concerns the way prediction methods are evaluated ---see, e.g., \cite{gama2013evaluating} for a discussion. In this paper, we advocate the use of the online error (also called sequential risk, prequential risk, or prequential error~\cite{gama2013evaluating}). This quantity measures the average of the errors made by the sequence of incrementally learned models, where one first tests the current model on the next example in the stream and then uses the same example to update the model. The sequential risk is therefore measured on each individual stream and does not specifically require stochastic assumptions on the way the stream is generated.

In this paper, we propose a novel incremental and nonparametric approach for the classification of data streams. We present four different instances of our approach (called BASE, BASE-ADJ, AUTO, and AUTO-ADJ) characterized by an increasing degree of adaptivity to the data. In particular, AUTO-ADJ is fully parameterless, a feature especially important in streaming settings where tuning is a hard task. Even though our algorithms are instance-based like nearest neighbour, the learned models are significantly smaller than those produced by competing baselines and more accurate when the online performance is measured against the model size. Finally, our methods (except BASE) are natively multiclass and can dynamically accommodate new classes as they appear in the stream.

In a nutshell, our algorithms work by incrementally covering the input space with balls of possibly different radii. Each new example that falls outside of the current cover becomes the center of a new ball. Examples are classified according to NN over the ball centers, where each ball predicts according to the majority of the labels of previous examples that fell in that ball. The set of balls is organized in a tree structure~\cite{Krauthgamer:2004:NNS:982792.982913}, so that predictions can be computed in time logarithmic in the number of balls. In order to increase the ability of the model to fit new data, the radii of the balls shrink, thus making room for new balls. The shrinking of the radius may depend on time or, in the more sophisticated variants of our algorithms, on the number of classification mistakes made by each ball classifier. Similarly to decision trees, where leaves are split according to their impurity, our method locally adapts the complexity of the model by allocating more balls in regions of the input space where the stream is harder to predict. A further improvement concerns the relocation of the ball centers in the input space: as our methods are completely incremental, the positioning of the balls depends on the order of the examples in the stream, which may result in a model using more balls than necessary. In order to mitigate this phenomenon, while avoiding a costly global optimization step to reposition the balls, we also consider a variant in which a K-means step is used to move the center of a ball being updated towards the median of the data points that previously fell in that ball. A further modification which we consider is aimed at keeping the model size bounded even in the presence of an arbitrarily long stream. This is achieved by introducing a randomized mechanism for discarding balls when the size bound is reached. Specifically, the mechanism discards a ball with probability proportional to the mistake rate of the ball classifier. The underlying idea is to get rid of the model parts that contribute the most to the global error and may replaced by a better arrangement of balls.

In summary, we introduce a simple and flexible approach for nonparametric classification of data streams. Our approach is fully modular: we predict using majority voting, but a fully trainable classifier could be used instead. The simplest version of our approach, applicable to streams with binary labels, enjoys strong theoretical guarantees: its mistake rate on any arbitrary stream converges to that of the best classification function that satisfies a certain regularity condition. The more complex versions of our approach learn multiclass classifiers without knowning the number of distinct labels in advance. We empirically show that our methods are excellent at trading-off classification accuracy with model size. Our most sophisticated method is fully parameterless. Finally, we show that a simple modification of our approach allows to keep the model size bounded, outperforming other methods measured at the same value of model size.

The rest of the paper is organized as follows. Section~\ref{sc:related} discusses related work. In Section~\ref{sc:setting}, we define the problem setting. In Section~\ref{sc:abc}, we present our nonparametric classification approach. In Section~\ref{sc:ci}, we discuss the theoretical properties of our approach and derive a formal performance guarantee for the simplest algorithm. We then introduce three more sophisticated versions that are empirically more effective. In Section~\ref{sc:exp}, we test the behaviour of our algorithms against state-of-the-art baselines. In Section~\ref{sec:fix_bud}, we introduce a simple modification of our approach to keep the model size bounded. Finally, Section~\ref{s:concl} concludes the paper. 

%NCB
%\input{related.tex}

\section{Related Work}
\label{sc:related}
Within the vast area of stream mining~\cite{gaber2007survey}, we focus our analysis of related work on the subarea that is most relevant to this study: nonparametric methods for stream classification. The most important approaches in this domain are:

\noindent
\textbf{Incremental decision and rule tree} learning systems, such as Very Fast Decision Tree (VFDT)~\cite{domingos2000mining} and Decision Rules (RULES)~\cite{gama2011learning} which use an incremental version of the split function computation ---see also~\cite{rutkowski2012decision,matuszyk2013correcting,duda2014novel,de2015splitting}.

\noindent
\textbf{Incremental variants of NN}, such as Condensed Nearest Neighbour (CNN)~\cite{wilson2000reduction} that stores only the misclassified instances, Lazy-Tree (L-Tree)~\cite{zhang2011enabling} condensing historical stream records into compact exemplars, and IBLStreams~\cite{shaker2012iblstreams}, an instance-based learning algorithms removing outliers or examples that have become redundant.	

\noindent	
\textbf{Incremental kernel-based} algorithms, such as the kernel Perceptron~\cite{freund1999large} with Gaussian kernels.\footnote{Gaussian kernels are universal~\cite{steinwart2002influence}, meaning that a kernel-based model can approximate any continuous classification function. Hence, algorithms using Gaussian kernels can be viewed as instance-based nonparametric learning algorithms.}

Note that our methods do not belong to any of the above three families: they do not perform a recursive partition of the feature space as decision trees, they do not allocate (or remove) instances based on the heuristics used by IBLStreams, and they do not use kernels.

As we explain next, our most basic algorithm is a variant for classification tasks of the algorithm proposed in~\cite{conf/nips/KpotufeO13} for nonparametric regression in a streaming setting. A similar algorithm was previously proposed in~\cite{hazan2007online} and analyzed without resorting to stochastic assumptions on the stream generation. A preliminary instance of our approach, without any theoretical analysis, was developed in~\cite{de2014online} for an action recognition application in video feeds.

%NCB
%\input{setting.tex}

\section{Problem Setting}
\label{sc:setting}
Our analysis applies to streams of data points belonging to an arbitrary metric space and depends on the \emph{metric dimension} of data points in the stream.
This notion of dimension extends to general metric spaces the traditional notions of dimension (e.g., Euclidean dimension and manifold dimension) \cite{C:74}. The metric dimension of a subset $S$ of a metric space $(\X, \rho)$ is $d$ if there exists a constant $C_S > 0$ such that, for all $\epsilon>0$, $S$ has an $\epsilon$-cover of size at most $C_S\epsilon^{-d}$ (an $\epsilon$-cover is a set of balls of radius $\epsilon$ whose union contains $S$). In practice, the metric dimension of the stream may be much smaller than the dimension of the ambient space $\X$. This is especially relevant in case of nonparametric algorithms, which typically have a bad dependence on the dimensionality of the data. Note that our algorithms do not require knowledge of $d$: the metric dimension of the stream is automatically estimated from the data.

The learner receives a sequence $(\bx_1,y_1),(\bx_2,y_2),\dots$ of examples, where each data point $\bx_t \in \X$ is annotated with a label $y_t$ from a set $\sY = \{1,\dots,K\}$ of possible class labels, which may change over time. The learner's task is to predict each label $y_t$ minimizing the overall number of prediction mistakes over the data stream. 

We derive theoretical performance guarantees for BASE, the simplest algorithm in our family (Algorithm~\ref{alg:base}), without making stochastic assumptions on the way the examples in the stream are generated. Note that this is a very strong type of guarantee: our results hold on \emph{any} individual stream of annotated data points.

%FRA
%\input{algo.tex}

\section{Adaptive Ball Covering}
\label{sc:abc}

%\section{Incremental Tree-based Classification}
The adaptive ball covering at the roots of our method was previously used in a theoretical work~\cite{conf/nips/KpotufeO13}.
Here, we distillate the main ideas behind that approach in a generic algorithmic approach (the template Algorithm~\ref{alg:main}) called ABACOC (Adaptive BAll COver for Classification). We then present our methods as specific instances of this generic template.

\begin{algorithm}[t]
	\caption{ABACOC TEMPLATE}
	\label{alg:main}
	\begin{algorithmic}[1]
		\Require metric $\rho$
		\State Initialize set of ball centers $\sS=\emptyset$
		\State \texttt{InitProcedure()}
		\For{$t=1,2,\dots$}
		\State Get input example $(\bx_t,y_t)$
		\If{$y_t \notin \sY$}
		\State Set $\sY = \sY \cup \{y_t\}$   $//$ add new class on the fly 
		\EndIf
		% %\If{$\sS \equiv \emptyset$ \OR $|\sS|=1$}
		% \If{$\sS \equiv \emptyset$}
		% \State Output default prediction $\yhat_t$
		% \State \texttt{InitProcedure()}
		% \Else
		\State Let $\mathcal{B}(\bx_s,\ve_s)$ be the ball in $\sS$ closest to $\bx_t$ 
		\State \texttt{OuputPrediction}$(\mathcal{B}_s)$ 
		%Output prediction $\yhat_t$ equal to the majority of the labels in the ball $\bx_s$
		\If{$\rho\bigl(\bx_s,\bx_t) \le \ve_s$}
		\State $\mathcal{B}$=\texttt{UpdateBallInformation}$(\mathcal{B}_s,(\bx_t,y_t))$
		\Else
		\State $\mathcal{B}$=\texttt{AddNewBall}$(\sS,\bx_s,(\bx_t,y_t))$ 
		\EndIf
		%\State Updates the counts of the label for each class in the ball $\bx_s$ using $y_t$
		\State \texttt{UpdateEpsilon}($\mathcal{B}$)
		%\EndIf	    		
		\EndFor
	\end{algorithmic}
\end{algorithm}

\subsection{The BASE Algorithm}
\label{sc:ci}
Our first instance of ABACOC is BASE (Algorithm~\ref{alg:base}), a randomized variant for binary classification of the ITBR (Incremental Tree-Based Regressor) algorithm proposed in~\cite{conf/nips/KpotufeO13}. BASE shrinks the radius (line 28) of the balls depending on (1) an estimate of the metric dimension of the stream and (2) the number of data points so far observed from the stream. This implies that the radii of all the balls shrink at the same rate. In the prediction phase, the ball nearest to the input example is considered and a randomized binary prediction is made based on the class distribution estimate locally computed in the ball. Laplace estimators (line 5) and randomized predictions (lines 6--8) are new features of BASE that were missing in ITBR.

We now analyze the performance of BASE using the notion of \emph{regret}~\cite{cesa2006prediction}. The regret of a randomized algorithm is defined as the difference between the expected number of classification mistakes made by the algorithm over the stream and the expected number of mistakes made by the best element in a fixed class of randomized classifiers. A randomized binary classifier is a mapping $f: \X \rightarrow [0,1]$, where $f(\bx)$ is the probability of predicting label $+1$. We consider the class $\mathcal{F}_L$ of $L$-Lipschitz predictors $f: \X \rightarrow [0,1]$ w.r.t.\ the metric $\rho$ of the space. Namely,
\[
\forall x, x'\in \X,\quad \abs{f(\bx) - f(\bx')} \leq L \, \distc{\bx, \bx'}~.
\]
Hence, a predictor is Lipschitz if, when we perturb the data point $\bx$, the prediction changes by an amount linear in the perturbation size. Lipschitz functions are a standard reference in the analysis of nonparametric algorithms.

The regret of BASE generating randomized predictions $\yhat_t$ is defined by (see also~\cite{hazan2007online})
\[
R_L(T) = \sum_{t=1}^T \Pr(\hat{y}_t\neq y_t) - \min_{f \in \mathcal{F}_L} \sum_{t=1}^T\Pr(f(x_t)\neq y_t)~.
\]
%Throughout the paper, and without loss of generality, we also assume that the space $\X$ has diameter at most $1$ under metric $\rho$.
For the BASE algorithm we can prove the following regret bound against \emph{any} Lipschitz randomized classifier, without \emph{any} assumption on the way the stream is generated. Moreover, similarly to ITBR, the regret upper bound depends on the unknown metric dimension $d$ of the space, automatically estimated by the algorithm.
\bookbox{
	\begin{theorem}
		\label{theo:regret}
		Fix a metric $\rho$ and any stream $(\bx_t,y_t)$ $t=1,\dots,T$ of binary labeled points $S = \{\bx_1,\dots,\bx_T\}$ in a metric space $(\X,\rho)$ of diameter $1$ and let $d$ be the metric dimension of $S$. Assume that Algorithm~\ref{alg:base} is run with parameter $\hat{C} \ge C_S$, where $C_S$ is such that $C_S\epsilon^{-d}$ upper bounds the size of any $\epsilon$-cover of $S$. Then, for any $L > 0$ we have
		\begin{align*}
		R_L(T)\leq 1.26 \left(2.5\sqrt{C_S} \,2^{d} + 1.5 L\right) T^{\frac{1+d}{2+d}}.
		\end{align*}
	\end{theorem}
}
The proof is in the next Section~\ref{app:proof}. Note that the algorithm does not know $L$, hence the regret bound above holds for all values of $L$ simultaneously. This theorem tells us that BASE is not an heuristic, but rather a principled approach with a specific performance guarantee. The performance guarantee implies that, on any stream, the expected mistake rate of BASE converges to that of the best $L$-Lipschitz randomized classifier at rate of order $(2^d+L)T^{-1/(2+d)}$.

%Despite its strong theoretically guarantees, the BASE algorithm suffers from some drawbacks given by the its theory-guided design. In particular, the procedure used to create and maintin the balls is probably sub-optimal for practical applications. Also, as said above, the BASE algorithm is designed to be used for binary classification. Here we are interested in a multi-class classification scenario where the number of classes is not known a priori.

Next, we generalize the BASE algorithm to multiclass classification, and make some modifications aimed at improving its empirical performance.

\begin{algorithm}[t]
	\caption{BASE}
	\label{alg:base}
	\begin{algorithmic}[1]
		\Require $\hat{C}$ (space diameter)
		\Procedure{\texttt{InitProcedure}}{}
		%\State $\sS = \{\bx_t\}$ and use $y_i$ to initialize estimates $p_t$ via (\ref{eq:prediction})
		\State  $\sS=\emptyset$, $i=1$, $t_i=0$, and $d_i=1$
		\EndProcedure
		\Procedure{\texttt{OuputPrediction}}{$\mathcal{B}_s$}
		\State $q_s=\frac{m_s+1}{n_s+2}$ \Comment{laplace estimator of counts}
		\State Set $\gamma_s=\frac{1}{2\sqrt{n_s+2}}$
		\State Set $p_t=\begin{cases} 0 &\mbox{if } q_s<\frac{1}{2}-\gamma_s \\ 
		1 & \mbox{if } q_s>\frac{1}{2}+\gamma_s\\
		\frac{1}{2}+(q_s-\frac{1}{2}) / (2 \gamma_s) & \mbox{otherwise}. \end{cases}$
		\State Predict $\hat{y}_t=1$ with probability $p_t$ and $0$ otherwise.
		\EndProcedure
		\Procedure{\texttt{UpdateBallInformation}}{$\mathcal{B}_s,(\bx_t,y_t)$}
		\State $m_s=m_s+y_t$ \Comment{number of $y_t=1$ in the ball}
		\State $n_s=n_s+1$ \Comment{total number of points in the ball}
		\EndProcedure
		\Procedure{\texttt{AddNewBall}}{$\sS,\bx_s,(\bx_t,y_t)$}
		\If {$|\sS|+1 > \hat{C}2^{d_i}\ve_t^{-d_i}$} \Comment{dimension check}
		\State $\sS= \emptyset$ \Comment{start Phase $i+1$}
		\State $d_{i+1}=\bigl\lceil\log\bigl(\frac{|\sS|+1}{\hat{C}}\bigr)/\log(2/\ve_t)\bigr\rceil$ 
		\State $i=i+1$ 
		\State $t_i=0$
		\EndIf
		\State $\sS = \sS \cup \{\bx_t\}$
		\State $m_t=y_t$ \Comment{number of $y_t=1$ in the ball}
		\State $n_t=1$ \Comment{first point in the ball}
		\State $t_i=t_i+1$ \Comment{counts the time steps within phase $i$}
		\EndProcedure
		\Procedure{\texttt{UpdateEpsilon}}{}
		\State $//$ radius dependent on current time step
		\State $\ve_t=t_i^{-1/(2+d_i)}$
		\EndProcedure
	\end{algorithmic}
\end{algorithm}

\section{Proofs}
\label{app:proof}
We use the following well-known fact: if $p_t = \Pr(\yhat_t = 1)$ for predicting $y_t\in\bool$ using a randomized label $\yhat_t\in\bool$, then $\Pr(\yhat_t \neq y_t)=|y_t-p_t|$.

Even if our algorithm is different from ITBR, we can still use the following lemma from ITBR analysis~\cite{conf/nips/KpotufeO13}. In the following, we say that a phase ends each time condition in line 15 of BASE is verified and use $T_i$ to denote the time steps included in phase $i$. Finally, $\sS_i$ denotes the maximum number of balls used in phase $i$.
\begin{lemma}[\cite{conf/nips/KpotufeO13}]
	\label{lem:di}
	Suppose BASE is run with parameter $\hat{C}\geq C_S$.
	The following invariants hold throughout the procedure for all phases $i\geq 1$:
	\begin{itemize}
		\item $i \leq d_i \leq d$.
		\item For any $t\in T_i$ we have $\abs{\Xspl_i} \leq \hat{C} \,4^{d_i} \epsilon_t^{-d_i}$.
	\end{itemize}
\end{lemma}
Define $\ell_t(p_t)=|p_t-y_t|$. Unlike the analysis in \cite{conf/nips/KpotufeO13}, here we cannot use a bias-variance decomposition. So, the key in the proof is to decompose the regret in two terms with behaviour similar to the bias and variance terms in the stochastic setting.
\begin{lemma}
	\label{lemma:one_epoch}
	Let $d$ be the metric dimension of the set $S$ of data points in the stream.
	Assume that $\hat{C}\geq C_S$.
	Then, in any phase $i$ and for any $f \in \mathcal{F}_L$ we have that
	\begin{align*}
	\sum_{t\in T_i} \Bigl(\ell_t(p_t) - \ell_t\bigl(f(\bx_t)\bigr)\Bigr) 
	\leq \Bigl(2\sqrt{\hat{C}} \,2^{d_i+1} + 1.5 L\Bigr) n_i^{\frac{1+d_i}{2+d_i}}.
	\end{align*}
\end{lemma}
\begin{proof}
	We use the notation $\bx_t \rightarrow \bx_s$ to say that $\bx_t$ is assigned to a ball with center $\bx_s$.
	We also denote by $n(\bx_s)$ the number of points assigned to a ball of center $\bx_s$.
	Define 
	\[
	p_s^* = \argmin_{p \in [0,1]} \sum_{t \,:\, \bx_t \rightarrow \bx_s} \ell_t(p).
	\]
	For each $\bx_s$ in $\sS_i$, we proceed by upper bounding the error as a sum of two components
	\begin{align*}
	\nonumber
	\sum_{t\,:\,\bx_t \rightarrow \bx_s} & \Bigl(\ell_t(p_t) - \ell_t\bigl(f(\bx_t)\bigr)\Bigr)
	=
	\sum_{t\,:\,\bx_t \rightarrow \bx_s} \bigl(\ell_t(p_t) - \ell_t(p_s^*)\bigr)
	\\ &+
	\sum_{t\,:\,\bx_t \rightarrow \bx_s} \Bigl(\ell_t(p_s^*) - \ell_t\bigl(f(\bx_t)\bigr) \Bigr).
	\label{eq:bias_variance}
	\end{align*}
	Using the definition of $p_s^*$ and the Lipschitz property of $f$, we have
	\begin{align*}
	\ell_t(p_s^*) &- \ell_t(f(\bx_t))
	\leq
	\ell_t(f(\bx_s)) - \ell_t(f(\bx_t))  
	\\&
	\leq |f(\bx_s) - f(\bx_t)| \leq L\,\distc{\bx_s,\bx_t} \leq L\,\epsilon_t ~.
	\end{align*}
	The prediction strategy in each ball is equivalent to the approach followed in \cite{FederMG92} (see also Exercise 8.8 in \cite{cesa2006prediction}).
	The only important thing to note is that the first prediction of the algorithm in a ball is made using the probability of the closest ball, even if it is further than $\epsilon_t$, instead of at random as in the original strategy in \cite{FederMG92}. It is easy to see that this adds an additional $0.5$ to the regret stated in \cite{FederMG92}. So we have
	\begin{align*}
	\sum_{t\,:\,\bx_t \rightarrow \bx_s} \left(\ell_t(p_t) - \ell_t(p_s^*)\right) \leq \sqrt{n({\bx_s})+1}+1 \leq 2.5\sqrt{n({\bx_s})}~.
	\end{align*}
	Hence overall we have
	\begin{align*}
	\sum_{t\,:\,\bx_t \rightarrow \bx_s} \Bigl(\ell_t(p_t) - \ell_t\bigl(f(\bx_t)\bigr)\Bigr)
	\leq 2.5\sqrt{n({\bx_s})} + L \sum_{t\,:\,\bx_t \rightarrow \bx_s} \epsilon_t~.
	\end{align*}
	Summing over all the $\bx_s \in \sS_i$,  we have
	\begin{align*}
	&\sum_{t\in T_i} \Bigl(\ell_t(p_t) - \ell_t\bigl(f(\bx_t)\bigr)\Bigr) \\
	&\quad \leq 2.5\sum_{s=1}^{|\sS_i|} \sqrt{n({\bx_s})} + L \sum_{t\in T_i} \epsilon_t \\
	&\quad \leq  2.5 |\sS_i| \sqrt{\frac{1}{|\sS_i|}\sum_{s=1}^{|\sS_i|} n({\bx_s}) } + L \sum_{t\in T_i} \epsilon_t \\
	&\quad = 2.5 \sqrt{|\sS_i| n_i} + L \sum_{t\in T_i} \epsilon_t.
	\end{align*}
	To bound $|\sS_i|$ we use Lemma~\ref{lem:di}, while to bound the last term, we have
	\begin{align*}
	\sum_{t\in T_i}\epsilon_t 
	&= \sum_{t=1}^{n_i} t^{-\frac{1}{2+d_i}} 
	\leq \int_{0}^{n_i} \tau^{-\frac{1}{2+d_i}}\, d\tau
	= \frac{d_i+2}{d_i+1} n_i^{\frac{d_i+1}{2+d_i}} \\
	&\leq 1.5 n_i^{\frac{d_i+1}{2+d_i}}
	\end{align*}
	where $n_i = |T_i|$.
	Overall we have
	\begin{align*}
	\sum_{t\in T_i} & \Bigl(\ell_t(p_t) - \ell_t\bigl(f(\bx_t)\bigr)\Bigr) \\
	&\leq 2.5 \sqrt{\hat{C}} \,2^{d_i} n_i^{\frac{d_i}{2(2+d_i)}+\frac{1}{2}} + 1.5 L n_i^{\frac{d_i+1}{2+d_i}}\\
	& = (2.5\sqrt{\hat{C}} \,2^{d_i} + 1.5 L) n_i^{\frac{1+d_i}{2+d_i}} \\
	& \le (2.5\sqrt{\hat{C}} \,2^{d} + 1.5 L) n_i^{\frac{1+d}{2+d}}~.
	\end{align*}
\end{proof}
We finish with the proof of Theorem~\ref{theo:regret}.
\begin{proof}
	Let $I$ denote the number of phases up to time $T$. 
	Let $B \triangleq 2.5\sqrt{\hat{C}} \,2^{d} + 1.5 L$.
	We use Lemma~\ref{lemma:one_epoch} in each phase and sum over the phases, to have
	\begin{align*}
	&\sum_{t=1}^T \left(\ell_t(p_t) - \ell_t(p_s^*)\right)
	= \sum_{i = 1}^I \sum_{t\in T_i} \left(\ell_t(p_t) - \ell_t(p_s^*)\right) \\
	&\quad \leq B\sum_{i = 1}^I n_i^{\frac{1+d}{2+d}}
	= B\, I \sum_{i = 1}^I \frac{1}{I} n_i^\frac{1+d}{2+d}
	\leq  B\, I \paren{\sum_{i = 1}^I \frac{n_i}{I}}^\frac{1+d}{2+d} \\
	&\quad = B\, I \paren{ \frac{T}{I}}^\frac{1+d}{2+d}
	\leq B d^\frac{1}{2+d} T^{\frac{1+d}{2+d}}
	\leq 1.26 B\, T^{\frac{1+d}{2+d}}
	\end{align*}
	where in the second inequality we use Jensen's inequality, and in the second to last inequality the first statement of Lemma~\ref{lem:di}.
\end{proof}

\subsection{The BASE algorithm with ball adjustment}
\begin{algorithm}[t]
	\caption{BASE-ADJ (BASE with ball adjustment)}
	\label{alg:base-adj}
	\begin{algorithmic}[1]
		\Require $\hat{C}$ (space diameter)
		\Procedure{\texttt{OuputPrediction}}{$\mathcal{B}_s$}
		\State $n_s = n_s(1)+\cdots+n_s(K)$ \Comment{total class counts}
		\State $p_s(k) = \frac{n_s(k)}{n_s} \qquad k=1,\dots,K$
		\State Predict ${\displaystyle \yhat_t = \argmax_{k\in\sY} p_s(k)}$
		\EndProcedure
		\Procedure{\texttt{UpdateBallInformation}}{$\mathcal{B}_s,(\bx_t,y_t)$}
		\State $//$ update ball centre on correct prediction
		\If{${\displaystyle y_t = \yhat_t}$}
		\State $\Delta=\bx_t-\bx_s; n_s=n_s+1;$  
		\State $\bx_s=\bx_s+\Delta/n_s$
		\EndIf
		\State Updates label counts $n_s(1),\ldots,n_s(K)$ in the ball $\mathcal{B}_s$ using $y_t$
		\EndProcedure
	\end{algorithmic}
\end{algorithm}
A natural way of generalizing the BASE algorithm to the multiclass case is by estimating the class probabilities in each ball. Note that this approach is naturally incremental w.r.t.\ the number of classes: new bins for counting are created on the fly as data points of new classes arrive.

Recall that the BASE algorithm greedly covers the input space. In particular, balls are always centered on input points. However, constraining the centers on data points is an intuitively sub-optimal strategy: it might be possible to cover the same region with a smaller number of balls if we could freely move their centers. As a full optimization of the position of the centers is not realistic in a streaming scenario, we introduce the BASE-ADJ variant which makes a partial optimization by using a step of the K-means algorithm~\cite{macqueen1967some}. More precisely, BASE-ADJ (Algorithm~\ref{alg:base-adj}, only the main changes w.r.t.\ BASE are shown) moves the center of each ball towards the average of the correct classified data points falling into it. In this way, the center of the ball tends to move towards the centroid of a cluster of points of a certain class. We expect this variant to generate less balls and also to have a better empirical performance.

We drop from BASE-ADJ the Laplace correction of class estimates and the randomization in the computation of the predicted label. Although these ingredients were used in the theoretical analysis, we noticed that they do not significantly affect the empirical results. Hence, BASE-ADJ always predicts the class with the largest class probability estimate (majority voting on the collected labels) within the ball closest to the current data point.

\subsection{The AUTO algorithm: automatic radius}

\begin{algorithm}[t]
	\caption{AUTO and AUTO-ADJ}
	\label{alg:auto}
	\begin{algorithmic}[1]
		\Require $\hat{d}$
		\Procedure{\texttt{InitProcedure}}{}
		\State $//$ wait until at least two different labels fed
		\If{$\sS \equiv \emptyset$}
		\State $\sS = \{\bx_1\}$ and initialize label counts
		\ElsIf {$y_t \neq y_1$} \Comment{$|\sS|=1$}
		\State $\sS = \sS \cup \{\bx_t\},\ve_1=\ve_t=\rho(\bx_1,\bx_t)$  \State Initialize label counts
		\Else
		\State \textbf{continue}
		\EndIf
		\EndProcedure
		\Procedure{\texttt{OuputPrediction}}{$\mathcal{B}_s$}
		\State $n_s = n_s(1)+\cdots+n_s(K)$ \Comment{total class counts}
		\State $p_s(k) = \frac{n_s(k)}{n_s} \qquad k=1,\dots,K$
		\State Predict ${\displaystyle \yhat_t = \argmax_{k\in\sY} p_s(k)}$
		\EndProcedure
		\Procedure{\texttt{UpdateBallInformation}}{$\mathcal{B}_s,(\bx_t,y_t)$}
		\State $//$ shrink radius on errors
		\If{${\displaystyle y_t \neq \yhat_t}$}
		\State Set $m_s = m_s + 1$ \Comment{update mistakes count}
		\ElsIf {AUTO-ADJ method}
		\State $//$ update ball centre if correct prediction
		\State $\Delta=\bx_t-\bx_s; u_s=u_s+1;$  
		\State $\bx_s=\bx_s+\Delta/u_s$ 
		\EndIf
		\State Updates label counts $n_s(1),\ldots,n_s(K)$ in the ball $\mathcal{B}_s$ using $y_t$
		\EndProcedure
		\Procedure{\texttt{AddNewBall}}{$\sS,\bx_s,(\bx_t,y_t)$}
		\State $\sS = \sS \cup \{\bx_t\}$, $R_t = \rho(\bx_t,\bx_s)$
		\State $m_t=0$ \Comment{ball mistakes count}
		\State $u_t=1$ \Comment{center updates count (for AUTO-ADJ)}
		\State Initialize label counts $n_s(1),\ldots,n_s(K)$ in the ball $\mathcal{B}_t$ using $y_t$
		\EndProcedure
		\Procedure{\texttt{UpdateEpsilon}}{$\mathcal{B}_s$}
		\State $//$ radius dependent on mistakes
		\State $\ve_s = R_s\, m_s^{-1/(2+\hat{d})}$
		\EndProcedure
	\end{algorithmic}
\end{algorithm}

One of the biggest issues with BASE (and ITBR) is the use of a common radius for all the balls.
In fact, in line 28 of Algorithm~\ref{alg:base} we have that the radii $\ve_s$ shrink uniformly with time $t$ at rate $t^{-1/(d_i+2)}$, where $d_i$ is the estimated metric dimension. However, we would like the algorithm to use smaller balls in regions of the input space where labels are more irregularly distributed and bigger balls in easy regions, where labels tend to be the same. 

In order to overcome this issue, in this section we introduce two other instances of ABACOC: AUTO and AUTO-ADJ. In these variants we let the radius of each ball shrink at a rate depending on the number of mistakes made by each local ball classifier, lines 20 and 36 in Algorithm~\ref{alg:auto}. Moreover, in order to get rid of the parameter $\hat{C}$ used to estimate the metric dimension, we initialize the radius of each ball to the distance to its closest ball, line 29 in Algorithm~\ref{alg:auto}. In other words, everytime a new ball is added its radius is set equal to the distance to the nearest already-existing ball.

AUTO-ADJ differs from AUTO because it implements the same strategy, introduced in BASE-ADJ, for updating the position of the centers. Note that this strategy, coupled with the shrinkage depending on the number of mistakes, makes a ball stationary once it is covering a region of the space that contains data points always annotated with the same label.

Using balls of different radii makes it impossible to work with the automatic estimate of the metric dimension used in BASE, BASE-ADJ and ITBR. For this reason, we further simplify the algorithms by resorting to a fixed estimate $\hat{d}$ of the intrinsic dimension $d$ as an input parameter.

\section{Experiments}
\label{sc:exp}

In this section, we describe baselines and datasets used in the experiments and report on the obtained results.
We conducted an extensive evaluation on standard machine learning datasets for the streaming setting. Generally, in real applications for high-speed data streams, when the system cannot afford to revise the current model after each observation of a data point, stream sub-sampling is used to keep the model size and the prediction efficiency under control. 
In order to emphasize the distinctive features of our approaches (i.e., good trade-off between accuracy and model size), we tested the online (prequential) performance using sub-sampling ---see Algorithm~5. In this setting, the algorithms have access to each true class label only with a certain probability. By varying this probability, we can explore different model sizes for each baseline algorithm and compare the resulting performances. Note also that, while in this work we only consider random sub-sampling, different and more active sampling schedules could be also envisioned. 

\subsection{Baseline and datasets}
\label{sc:base_data}
We considered eleven popular datasets for stream mining
%\footnote{We considered datasets with high number of instances and not too many dimensions. nonparametric models generally suffer very high dimensional data.} 
listed in Table~\ref{tab:data}. 

\begin{table}[h]
	\begin{center}
		\begin{tabular}{@{}lccccc@{}}
			\toprule
			\textbf{Data}  & \textbf{Cls} & \textbf{Dim} & \textbf{Examples} & \textbf{Drift} & \textbf{Source} \\ \midrule
			sensor           & 54             & 5                & 2,219,803                 & no & SDMR          \\
			kddcup99   & 23               & 41                 & 494,021                  & no & SDMR          \\
			powersupply          & 24               & 2                & 29,928                  & yes & SDMR          \\
			hyperPlane           & 5               & 10                & 100,000                   & yes & SDMR          \\
			sea      & 2               & 3                 & 60,000                   & yes & DF       \\
			poker          & 10               & 10                 & 25,010                  & no & MOA             \\
			covtype        & 7                & 54                 & 581,012                 & yes & MOA             \\
			airlines       & 2                & 608                & 539,383                 & yes & MOA             \\
			electricity    & 2                & 8                  & 45,312                  & yes & MOA             \\
			connect-4      & 3                & 126                & 67,557                  & no & LIBSVM          \\
			acoustic       & 3                & 50                 & 78,823                  & no & LIBSVM          \\
			\bottomrule
		\end{tabular}
	\end{center}
	\caption{Datasets used for benchmarking.}
	\label{tab:data}
\end{table}
As indicated in the table, datasets are from the Stream Data Mining repository (SDMR)~\cite{SDMR}, the Data Sets with Concept Drift repository (DF)~\cite{DF}, the Massive Online Analysis (MOA) collection\footnote{\url{moa.cms.waikato.ac.nz/datasets/}}, and the LIBSVM classification repository\footnote{\url{www.csie.ntu.edu.tw/~cjlin/libsvmtools/datasets/}}. In all experiments, we measured the \textsl{online accuracy} (prequential error in~\cite{gama2013evaluating} or ``Interleaved Test-Then-Train'' validation in MOA\footnote{\url{moa.cms.waikato.ac.nz/}}). This is the average performance when each new example in the stream is predicted using the classifier trained only over the past examples in the stream ---see Algorithm~5 (line 6).

\begin{algorithm}[t]                   % enter the algorithm environment
	\caption{Online sub-sampling evaluation protocol}          % give the algorithm a caption
	\label{alg:online-prot}  
	\begin{algorithmic}[1]                     % enter the algorithmic environment
		\Require  $\texttt{rate}$, Stream $(\bx_1,y_1),(\bx_2,y_2),\ldots$
		\State Initialize online accuracy $M_0 = 0$
		\For{$t=1,2,\ldots$}
		\State Receive instance $\bx_t$ from stream
		\State Compute class label prediction $\yhat_t$
		\State Receive true class label $y_t$ 
		\State Update $M_t = \bigl(1-\tfrac{1}{t}\bigr)M_{t-1} + \tfrac{1}{t}\mathbb{I}\{\yhat_t = y_t\}$
		%	\If {BUDGET setting}
		\If {$\texttt{rand()}<\texttt{rate}$}
		\State Update model with new example $(\bx_t,y_t)$
		\EndIf 	
		%	\Else
		%	    \State $//$ FULL setting
		%		\State Update model with new example $(\bx_t,y_t)$
		%	\EndIf  	
		\EndFor
	\end{algorithmic}
\end{algorithm}

In a pre-processing phase, the categorical attributes were binarized.
BASE and BASE-ADJ received normalized input instances (Euclidean norm) allowing the input parameter $\hat{C}$ (space diameter) to be set to 1.
% in order to permit an efficient computation of the different learners.
We compared our ABACOC methods BASE\footnote{We used the multiclass version as for BASE-ADJ.} (Algorithm~2), BASE-ADJ (Algorithm~3), AUTO and AUTO-ADJ (Algorithm~4) against some of the most popular incremental nonparametric baselines (see Section~\ref{sc:related}) in the stream mining literature: K-NN with parameter $K=3$ (NN3) (see next paragraph for a justification of this choice), Condensed Nearest Neighbor~\cite{wilson2000reduction} (CNN), a streaming version of NN which only stores mistaken points, the multiclass Perceptron with Gaussian kernel~\cite{crammer2003ultraconservative} (K-PERC), a decision tree algorithm for streaming data~\cite{domingos2000mining} (VDFT), and a recent algorithm for learning decision rules on streaming data~\cite{gama2011learning} (RULES). For VDFT and RULES we used the implementation available in MOA, while K-PERC was run using the code in DOGMA~\cite{Orabona09}. The ABACOC algorithms were implemented in MATLAB\footnote{code available at \url{http://mloss.org/software/view/560/}.}. We did not consider the L-Tree~\cite{zhang2011enabling} and IBLStreams~\cite{shaker2012iblstreams} methods described in Section~\ref{sc:related} as L-Tree is an efficient approximation of NN (outperformed by NN, see~\cite{zhang2011enabling}) and IBLStreams never performs better than RULES (both implemented in MOA) on our datasets.
%\footnote{The MATLAB version in DOGMA and the JAVA version in MOA extensions \url{http://moa.cms.waikato.ac.nz/moa-extensions/}.}.

Where necessary, the parameters of the competitor methods were individually tuned on each dataset using an algorithm-specific grid of values in order to obtain the best online performance. Hence, the results of the competitors are not worse than the ones obtainable with a tuning of the parameters using standard cross-validation methods. For our methods, we used the Euclidean distance as metric $\rho$. Based on preliminary experiments, we noticed that the parameter $\hat{d}$ does not affect significantly the performance in AUTO and AUTO-ADJ, so we set it to $2$. With $\hat{d}$ fixed to this value, our methods are essentially parameterless, which is a very attractive feature in a streaming setting where cross-validation can not be easily applied.
\begin{figure*}[t]
	\subfigure[BASE]{
		\centering
		\includegraphics[width=0.22\linewidth]{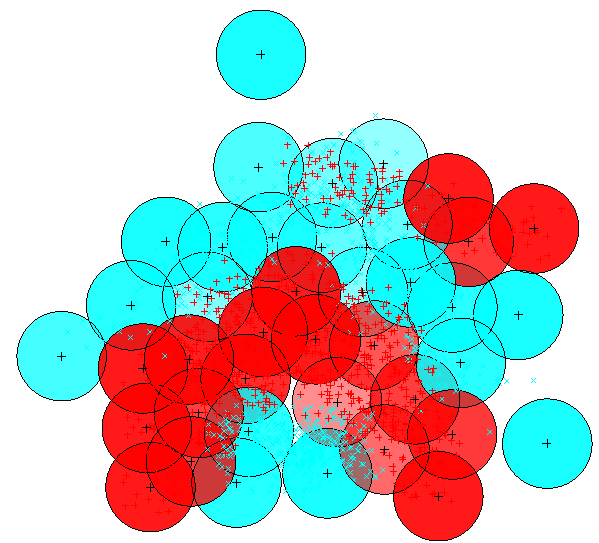}
	}
	\subfigure[BASE-ADJ]{
		\centering
		\includegraphics[width=0.22\linewidth]{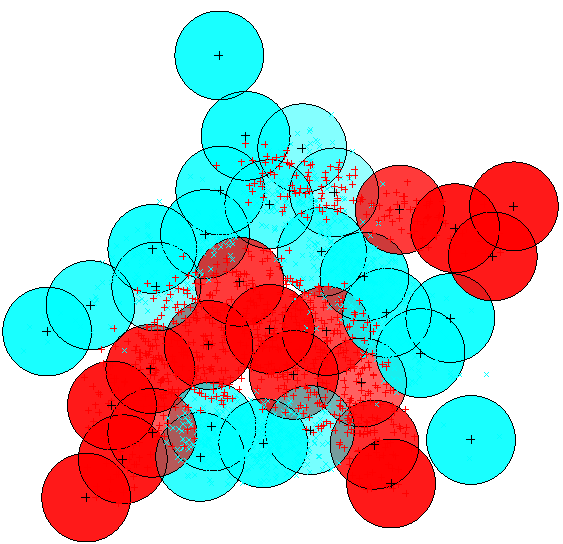}
	}
	\subfigure[AUTO]{
		\centering
		\includegraphics[width=0.22\linewidth]{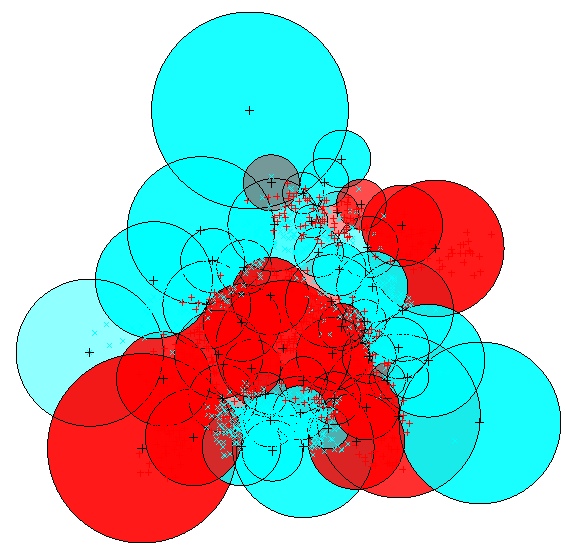}
	}
	\subfigure[AUTO-ADJ]{
		\centering
		\includegraphics[width=0.22\linewidth]{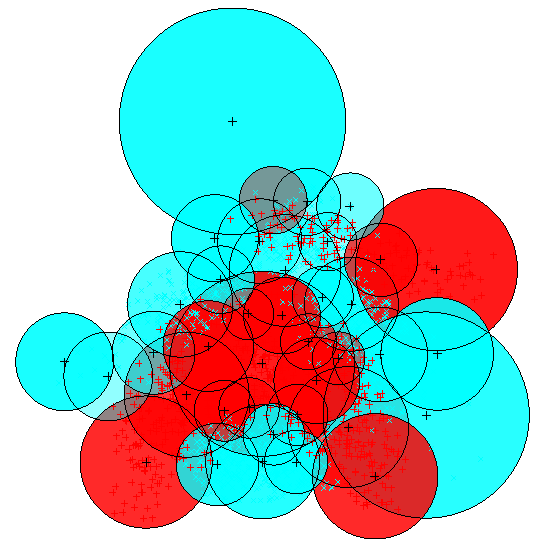}
	}\hfill                 
	\caption{Empirical behaviours of all versions of ABACOC algorithm on 2000 datapoints of the \emph{banana} dataset.
		The intensity of the colour of each ball is proportional to the conditional class probability of the two classes.
		\label{fig:ball_comp}
	}
\end{figure*}
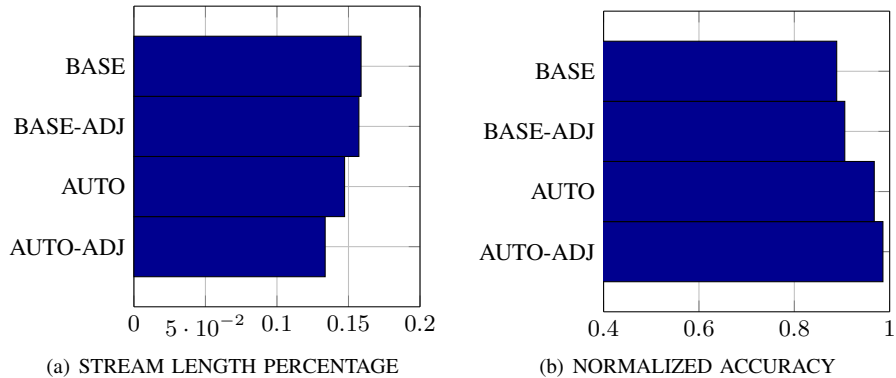
\begin{figure*}[t]
	\centering
	\subfigure[STREAM LENGTH PERCENTAGE]{
		\centering
% Minimal pgfplots version: 1.3
%  http://www.mathworks.com/matlabcentral/fileexchange/22022-matlab2tikz
\definecolor{mycolor1}{rgb}{0.00000,0.00000,0.56250}%
\begin{tikzpicture}

\begin{axis}[%
width=0.95092\figurewidtha,
height=\figureheighta,
at={(0\figurewidtha,0\figureheighta)},
scale only axis,
separate axis lines,
every outer x axis line/.append style={black},
every x tick label/.append style={font=\color{black}},
xmin=0,
xmax=0.2,
xmajorgrids,
xminorgrids,
every outer y axis line/.append style={black},
every y tick label/.append style={font=\color{black}},
ymin=0,
ymax=5,
ytick={1,2,3,4},
yticklabels={{AUTO-ADJ},{AUTO},{BASE-ADJ},{BASE}},
ymajorgrids,
yminorgrids
]

\addplot[area legend,solid,fill=mycolor1,draw=black,forget plot]
table[row sep=crcr] {%
x	y\\
0	0.5\\
0.133699375001397	0.5\\
0.133699375001397	1.5\\
0	1.5\\
}--cycle;

\addplot[area legend,solid,fill=mycolor1,draw=black,forget plot]
table[row sep=crcr] {%
x	y\\
0	1.5\\
0.147274938721572	1.5\\
0.147274938721572	2.5\\
0	2.5\\
}--cycle;

\addplot[area legend,solid,fill=mycolor1,draw=black,forget plot]
table[row sep=crcr] {%
x	y\\
0	2.5\\
0.157274938721572	2.5\\
0.157274938721572	3.5\\
0	3.5\\
}--cycle;

\addplot[area legend,solid,fill=mycolor1,draw=black,forget plot]
table[row sep=crcr] {%
x	y\\
0	3.5\\
0.158796358137774	3.5\\
0.158796358137774	4.5\\
0	4.5\\
}--cycle;

\end{axis}
\end{tikzpicture}%
	}
	\subfigure[NORMALIZED ACCURACY]{
		\centering
% Minimal pgfplots version: 1.3
%  http://www.mathworks.com/matlabcentral/fileexchange/22022-matlab2tikz
\definecolor{mycolor1}{rgb}{0.00000,0.00000,0.56250}%
\begin{tikzpicture}

\begin{axis}[%
width=0.95092\figurewidtha,
height=\figureheighta,
at={(0\figurewidtha,0\figureheighta)},
scale only axis,
separate axis lines,
every outer x axis line/.append style={black},
every x tick label/.append style={font=\color{black}},
xmin=0.4,
xmax=1,
xmajorgrids,
xminorgrids,
every outer y axis line/.append style={black},
every y tick label/.append style={font=\color{black}},
ymin=0,
ymax=5,
ytick={1,2,3,4},
yticklabels={{AUTO-ADJ},{AUTO},{BASE-ADJ},{BASE}},
ymajorgrids,
yminorgrids
]

\addplot[area legend,solid,fill=mycolor1,draw=black,forget plot]
table[row sep=crcr] {%
x	y\\
0	0.5\\
0.985820239222946	0.5\\
0.985820239222946	1.5\\
0	1.5\\
}--cycle;

\addplot[area legend,solid,fill=mycolor1,draw=black,forget plot]
table[row sep=crcr] {%
x	y\\
0	1.5\\
0.967612800640985	1.5\\
0.967612800640985	2.5\\
0	2.5\\
}--cycle;

\addplot[area legend,solid,fill=mycolor1,draw=black,forget plot]
table[row sep=crcr] {%
x	y\\
0	2.5\\
0.90582942545619	2.5\\
0.90582942545619	3.5\\
0	3.5\\
}--cycle;

\addplot[area legend,solid,fill=mycolor1,draw=black,forget plot]
table[row sep=crcr] {%
x	y\\
0	3.5\\
0.888895063124599	3.5\\
0.888895063124599	4.5\\
0	4.5\\
}--cycle;

\end{axis}
\end{tikzpicture}%
	}
	\caption{Model size and online performance averaged over all datasets in Table~\ref{tab:data} of our four methods. Performances are computed by normalizing each performance relative to the best performer for each dataset, and then averaging over the datasets.
		\label{fig:bin}
	}
\end{figure*}

\subsection{Comparison among our methods}
First, we compared the empirical behaviour of all our algorithms on the two-dimensional dataset \emph{banana},\footnote{\url{http://mldata.org/repository/data/viewslug/banana-ida/}} in Figure~\ref{fig:ball_comp}. The simplicity of this dataset allows us to show visually the difference between the four algorithms. BASE is seen to have many overlapping balls. On the other hand, AUTO has balls of different radii and not so overlapping. Finally, BASE-ADJ and AUTO-ADJ, the variants of BASE and AUTO that update the centers of the balls, have a smaller number of balls than BASE and AUTO respectively.
Also, note how the use of a varying shrinking radius in AUTO and AUTO-ADJ results in bigger balls that cover very large regions of the space.
To verify the intuition emerged from Figure~\ref{fig:ball_comp}, we empirically tested the performance of our methods on the entire benchmark of Table~\ref{tab:data}, running Algorithm~5 with $\texttt{rate}=1$. In Figure~\ref{fig:bin}(a), we show the resulting model sizes in terms of the stream length percentage used to represent the models (fraction of input samples used as ball centers) of each method averaged over all datasets in our benchmark suite. Figure~\ref{fig:bin}(b) shows the average normalized accuracy of each method as a fraction of the accuracy of the best-performing method on each dataset. Note that, due to the adjustment procedure added to BASE-ADJ and AUTO-ADJ, they use a small fraction of data to represent their models while achieving a performance better than, respectively, BASE and AUTO. Finally, we observe that AUTO-ADJ simultaneously achieves the smallest model and the best performance.

\subsection{Comparison against baselines}
\label{sec:base_comp}
We now turn to describing the sub-sampling experiments. In a streaming setting, the model size and thus the computational efficiency of the prediction system is a key feature. The goal of the experiments is to show the trade-off between online performance and model size for each algorithm.
The model size is measured by: the number of balls used to cover the feature space (ABACOC), the number of stored instances (K-PERC, NN, CNN), the number of leaves (VFDT) or rules (RULES) used to partition the feature space.
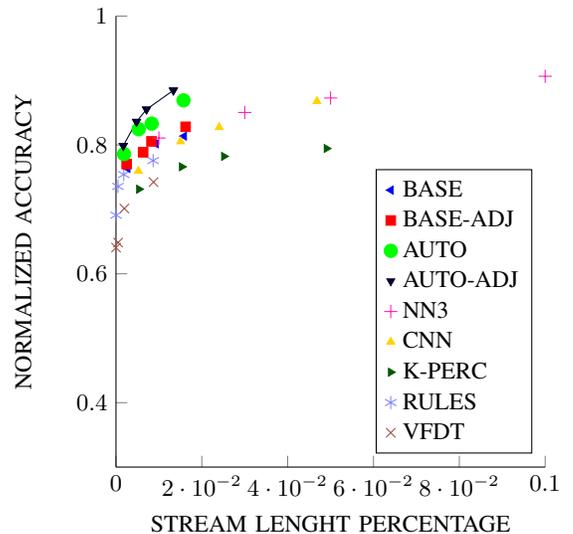
\begin{figure}[t]
	\centering
% Minimal pgfplots version: 1.3
%  http://www.mathworks.com/matlabcentral/fileexchange/22022-matlab2tikz
\definecolor{mycolor1}{rgb}{0.00000,0.00000,0.17241}%
\definecolor{mycolor2}{rgb}{1.00000,0.10345,0.72414}%
\definecolor{mycolor3}{rgb}{1.00000,0.82759,0.00000}%
\definecolor{mycolor4}{rgb}{0.00000,0.34483,0.00000}%
\definecolor{mycolor5}{rgb}{0.51724,0.51724,1.00000}%
\definecolor{mycolor6}{rgb}{0.62069,0.31034,0.27586}%
\begin{tikzpicture}

\begin{axis}[%
width=0.95092\figurewidthb,
height=\figureheightb,
at={(0\figurewidthb,0\figureheightb)},
scale only axis,
every outer x axis line/.append style={black},
every x tick label/.append style={font=\color{black}},
xmin=0,
xmax=0.1,
xlabel={STREAM LENGHT PERCENTAGE},
every outer y axis line/.append style={black},
every y tick label/.append style={font=\color{black}},
ymin=0.3,
ymax=1,
ylabel={NORMALIZED ACCURACY},
axis x line*=bottom,
axis y line*=left,
legend style={at={(0.97,0.03)},anchor=south east,legend cell align=left,align=left,draw=black}
]
\addplot [color=blue,mark size=1.7pt,only marks,mark=triangle*,mark options={solid,rotate=90,fill=blue}]
  table[row sep=crcr]{%
0.0158796358137774	0.813662777381809\\
};
\addlegendentry{BASE};

\addplot [color=red,mark size=1.8pt,only marks,mark=square*,mark options={solid,fill=red}]
  table[row sep=crcr]{%
0.0162236780436824	0.828232601166835\\
};
\addlegendentry{BASE-ADJ};

\addplot [color=green,mark size=2.5pt,only marks,mark=*,mark options={solid,fill=green}]
  table[row sep=crcr]{%
0.0157274938721572	0.869284645377653\\
};
\addlegendentry{AUTO};

\addplot [color=mycolor1,mark size=1.7pt,only marks,mark=triangle*,mark options={solid,rotate=180,fill=mycolor1}]
  table[row sep=crcr]{%
0.0133699375001397	0.884799792385142\\
};
\addlegendentry{AUTO-ADJ};

\addplot [color=mycolor2,mark size=2.5pt,only marks,mark=+,mark options={solid,fill=mycolor2}]
  table[row sep=crcr]{%
0.1	0.906635804824855\\
};
\addlegendentry{NN3};

\addplot [color=mycolor3,mark size=1.7pt,only marks,mark=triangle*,mark options={solid,fill=mycolor3}]
  table[row sep=crcr]{%
0.0468184706279372	0.867866194271921\\
};
\addlegendentry{CNN};

\addplot [color=mycolor4,mark size=1.7pt,only marks,mark=triangle*,mark options={solid,rotate=270,fill=mycolor4}]
  table[row sep=crcr]{%
0.0490601952807886	0.794570037908928\\
};
\addlegendentry{K-PERC};

\addplot [color=mycolor5,mark size=2.5pt,only marks,mark=asterisk,mark options={solid,fill=mycolor5}]
  table[row sep=crcr]{%
0.0086573519712239	0.775774933451997\\
};
\addlegendentry{RULES};

\addplot [color=mycolor6,mark size=2.5pt,only marks,mark=x,mark options={solid,fill=mycolor6}]
  table[row sep=crcr]{%
0.00874578247625048	0.742360700867513\\
};
\addlegendentry{VFDT};

\addplot [color=blue,mark size=1.7pt,only marks,mark=triangle*,mark options={solid,rotate=90,fill=blue},forget plot]
  table[row sep=crcr]{%
0.00928693137251816	0.801794724305936\\
};
\addplot [color=red,mark size=1.8pt,only marks,mark=square*,mark options={solid,fill=red},forget plot]
  table[row sep=crcr]{%
0.0082787791465653	0.805566364014253\\
};
\addplot [color=green,mark size=2.5pt,only marks,mark=*,mark options={solid,fill=green},forget plot]
  table[row sep=crcr]{%
0.00831493578721217	0.833169376810335\\
};
\addplot [color=mycolor1,mark size=1.7pt,only marks,mark=triangle*,mark options={solid,rotate=180,fill=mycolor1},forget plot]
  table[row sep=crcr]{%
0.00711201451614163	0.855287997611218\\
};
\addplot [color=mycolor2,mark size=2.5pt,only marks,mark=+,mark options={solid,fill=mycolor2},forget plot]
  table[row sep=crcr]{%
0.05	0.872887209474872\\
};
\addplot [color=mycolor3,mark size=1.7pt,only marks,mark=triangle*,mark options={solid,fill=mycolor3},forget plot]
  table[row sep=crcr]{%
0.0240456351341686	0.827564930207116\\
};
\addplot [color=mycolor4,mark size=1.7pt,only marks,mark=triangle*,mark options={solid,rotate=270,fill=mycolor4},forget plot]
  table[row sep=crcr]{%
0.0251479143539507	0.782328080473603\\
};
\addplot [color=mycolor5,mark size=2.5pt,only marks,mark=asterisk,mark options={solid,fill=mycolor5},forget plot]
  table[row sep=crcr]{%
0.0018114207160928	0.754101872302473\\
};
\addplot [color=mycolor6,mark size=2.5pt,only marks,mark=x,mark options={solid,fill=mycolor6},forget plot]
  table[row sep=crcr]{%
0.0019506265277739	0.701420393602566\\
};
\addplot [color=blue,mark size=1.7pt,only marks,mark=triangle*,mark options={solid,rotate=90,fill=blue},forget plot]
  table[row sep=crcr]{%
0.00669922394700927	0.788817150773859\\
};
\addplot [color=red,mark size=1.8pt,only marks,mark=square*,mark options={solid,fill=red},forget plot]
  table[row sep=crcr]{%
0.00632476439370535	0.788719204384367\\
};
\addplot [color=green,mark size=2.5pt,only marks,mark=*,mark options={solid,fill=green},forget plot]
  table[row sep=crcr]{%
0.00528172098591453	0.824447682954128\\
};
\addplot [color=mycolor1,mark size=1.7pt,only marks,mark=triangle*,mark options={solid,rotate=180,fill=mycolor1},forget plot]
  table[row sep=crcr]{%
0.00477804532125264	0.836187876405584\\
};
\addplot [color=mycolor2,mark size=2.5pt,only marks,mark=+,mark options={solid,fill=mycolor2},forget plot]
  table[row sep=crcr]{%
0.03	0.850471120421857\\
};
\addplot [color=mycolor3,mark size=1.7pt,only marks,mark=triangle*,mark options={solid,fill=mycolor3},forget plot]
  table[row sep=crcr]{%
0.0150387968473564	0.805350521258591\\
};
\addplot [color=mycolor4,mark size=1.7pt,only marks,mark=triangle*,mark options={solid,rotate=270,fill=mycolor4},forget plot]
  table[row sep=crcr]{%
0.0153259774340651	0.76613015301157\\
};
\addplot [color=mycolor5,mark size=2.5pt,only marks,mark=asterisk,mark options={solid,fill=mycolor5},forget plot]
  table[row sep=crcr]{%
0.000465196945481186	0.735694203264042\\
};
\addplot [color=mycolor6,mark size=2.5pt,only marks,mark=x,mark options={solid,fill=mycolor6},forget plot]
  table[row sep=crcr]{%
0.000493796134728456	0.648502213178515\\
};
\addplot [color=blue,mark size=1.7pt,only marks,mark=triangle*,mark options={solid,rotate=90,fill=blue},forget plot]
  table[row sep=crcr]{%
0.00271877178849802	0.763649009791074\\
};
\addplot [color=red,mark size=1.8pt,only marks,mark=square*,mark options={solid,fill=red},forget plot]
  table[row sep=crcr]{%
0.00248747924897035	0.770343793601516\\
};
\addplot [color=green,mark size=2.5pt,only marks,mark=*,mark options={solid,fill=green},forget plot]
  table[row sep=crcr]{%
0.00187243110481257	0.785947766197627\\
};
\addplot [color=mycolor1,mark size=1.7pt,only marks,mark=triangle*,mark options={solid,rotate=180,fill=mycolor1},forget plot]
  table[row sep=crcr]{%
0.00169876944075864	0.798865572849225\\
};
\addplot [color=mycolor2,mark size=2.5pt,only marks,mark=+,mark options={solid,fill=mycolor2},forget plot]
  table[row sep=crcr]{%
0.01	0.810757572226829\\
};
\addplot [color=mycolor3,mark size=1.7pt,only marks,mark=triangle*,mark options={solid,fill=mycolor3},forget plot]
  table[row sep=crcr]{%
0.00522725111883961	0.759686435985412\\
};
\addplot [color=mycolor4,mark size=1.7pt,only marks,mark=triangle*,mark options={solid,rotate=270,fill=mycolor4},forget plot]
  table[row sep=crcr]{%
0.00535352526259664	0.731363481186716\\
};
\addplot [color=mycolor5,mark size=2.5pt,only marks,mark=asterisk,mark options={solid,fill=mycolor5},forget plot]
  table[row sep=crcr]{%
2.05424435803479e-05	0.691000934039775\\
};
\addplot [color=mycolor6,mark size=2.5pt,only marks,mark=x,mark options={solid,fill=mycolor6},forget plot]
  table[row sep=crcr]{%
1.34113219080892e-05	0.640687628028311\\
};
\addplot [color=mycolor1,solid,mark size=1.7pt,mark=triangle*,mark options={solid,rotate=180,fill=mycolor1},forget plot]
  table[row sep=crcr]{%
0.00169876944075864	0.798865572849225\\
0.00477804532125264	0.836187876405584\\
0.00711201451614163	0.855287997611218\\
0.0133699375001397	0.884799792385142\\
};
\end{axis}
\end{tikzpicture}%
	\caption{
		Online performance against model size averaged over the datasets. The model size is relative to the stream length, whereas the online performance is measured relative to the top-performing method on each dataset without restriction on model size. 
		\label{fig:budget_summary}
	}
\end{figure}
% In the experiments we consider the final fraction of examples used by the algorithms as a complexity indicator performance ---i.e the number of balls of our methods divided by the total number of examples of the stream.
%Our experimental results are described using:
%\begin{enumerate}
%	\item A summary (Table~\ref{tab:data_perf1}) of the online performance on the streams generated from each dataset;
%	\item A bar plot (Figure~\ref{fig:bin}) summarizing the average performance and model size of each method on the full benchmark suite;
%	\item A plot (Figure~\ref{fig:budget_summary}) illustrating the behaviour of the algorithms in the budget setting. 
%\end{enumerate}

%This is especially apparent when comparing BALLs to BALLs-C ---see Figure~\ref{fig:ball_comp}(c) and Figure~\ref{fig:ball_comp}(d).
% A plausible explanation of this trend it may by due to the balls centres moving towards the class centroids, permitting the balls to make less mistakes (therefore bigger ball radii and fewer balls covering in pure density zone), and faster convergence to the real class distribution (better performances even in small datasets).      
We ran all the methods using values $\texttt{rate}=\{1\%,3\%,5\%,10\%\}$ and the same random seeds for all algorithms.\footnote{We remark that the $\texttt{rate}$ is only an upper bound on the model size. In fact, the methods can select a smaller fraction of data to represent the model.} In Figure~\ref{fig:budget_summary}, we plot the normalized online performance against model size, averaged over the datasets. The model size is relative to the stream length, whereas the online performance is measured relative to the top-performing method on each dataset without restriction on model size. 
As we can see from the plot, NN3 saturates the model size and achieves a slightly better overall performance on the larger model sizes. However, it suffers at low budget values and small model sizes. CNN works better than K-PERC and decision trees. VFDT and RULES use very little memory but have a worse performance than the other methods. BASE-ADJ improves on the performance of BASE. AUTO attains a better performance than BASE and AUTO-ADJ achieves the overall best trade-off between accuracy and model size. In fact, as we can see in Figure~\ref{fig:budget_summary}, the AUTO-ADJ curve dominates the other ones. Moreover, it attains $90\%$ of the best full-sampling methods while using only $1.5\%$ of the data to represent the model. Because of the better performance exhibited by our methods with respect to the baselines at the same model size values, we can infer that our methods have a better way of choosing the data points that define their models.

%Summarizing, AUTO appears to be a good choice at medium-high budget values, BASE-ADJ is good at medium budget values, and AUTO-ADJ seems to be the best method at small budget values.  

%We also conducted a statistical analysis to assess if the better behaviour of AUTO-ADJ in the budget setting is significant.
%We have used the nonparametric Friedman test, followed by a post-hoc test using the Holm step-down procedure, as prescribed in~\cite{Demsar06}. Note that this is the only procedure to test classifiers that is not based on assumptions that are known to be violated in comparing machine learning algorithms, see \cite{Demsar06} for details.
%In order to consider the variability in model size between the different algorithms, for each value of the model size we have
%considered for each algorithm the setting of the budget that gives the closest model size.
%In this way, in the range of the model size of $[0.01,0.25]$ AUTO-ADJ is the best algorithm. However, only in the range
%$[0.01,0.07]$ the difference is statistically significative with
%$p<0.13$. On the other hand, the p-value is less than 0.22 between
%$[0.08,0.12]$, and not significant outside of this range.
%Hence, we can say that AUTO-ADJ is significantly better than \emph{all} the other baselines in very low regimes of the budget, while it has the comparable performance of the best competitors in medium regimes.

\section{Constant model size}
\label{sec:fix_bud}
In this section we propose a simple method for making the memory footprint bounded, even in the presence of an arbitrarily long data stream. When the model size reaches a given limit, the algorithm starts to discard the examples supporting the model that are judged to be less informative for the prediction task. More precisely, it is reasonable to discard the local classifiers that are making the largest number of mistakes. This happens essentially for two reasons: 1) the optimal decision surface in that region is complex and/or the noise rate is high; 2) there is concept drift~\cite{tsymbal2004problem}, that is the optimal decision surface is locally changing over time. Removing local classifiers with a high mistake rate may then help because: we are discarding classifiers that are making essentially random decisions; moreover, we make room for new classifiers that rely on fresh statistics (good in case of concept drift) and are possibly better positioned to capture a complex decision surface. Thus, in order to curb the memory footprint, we propose a simple approach based on deleting existing balls whenever a given budget parameter is attained. This is crucial for real-time applications, as NN search in the prediction phase is logarithmic on the number of balls. The probability of deleting any given ball is proportional to the number of mistakes made so far by the associated classifier.
% In fact, if a ball are making many mistakes, it is not useful to use his class probability estimate to summing up the class scores during the prediction phase, as the ball probability estimate approaching to the uniform distribution for the class seen so far by the ball.
Namely, after the budget is reached, whenever a new ball is added an existing ball $i$ is discarded according to the Laplace-corrected probability
\begin{equation} \label{eq:prob_del}
%\[
\Pr(i\,\text{discarded}) =  \frac{m_i + 1}{\sum_{j \in \sS}{m_j}+|\sS|}
%\]
%\frac{1}{\sum_{t=1}^{T_i-1} w_{\pi(t)}}
\end{equation}
where $m_i$ is the number of mistakes made by ball $i \in \sS$. 
We run the experiments in the same setting of Section~\ref{sec:base_comp}, where we did not make any restriction on the sub-sampling rate ($\texttt{rate}=1$ in Algorithm~5). We added to AUTO-ADJ a constant model size bound. With respect to sub-sampling, here the algorithm has more control over the data points that support the model. We report in Table~\ref{tab:data_perf1} and in Table~\ref{tab:data_perf2} the performance with budget $10\%$ and $1\%$ of the method AUTO-ADJ with constant budget, called AUTO-ADJ FIX, compared to NN3 and AUTO-ADJ which performed the best in the previous experiments using the same final model sizes.  
\begin{table}[h]
	\begin{center}
		\begin{tabular}{@{}lccc@{}}
			\toprule
			\textbf{Data} & \textbf{NN3}       & \textbf{AUTO-ADJ}  & \textbf{AUTO-ADJ FIX} \\ \midrule
			kddcup99      & .714$|$.100          & .614$|$\textbf{.010}          & \textbf{.792}$|$.069    \\
			poker         & .677$|$.100          & .710$|$\textbf{.003}          & \textbf{.719}$|$.036    \\
			connect-4     & .592$|$.100          & .605$|$\textbf{.011}          & \textbf{.635}$|$.026    \\
			acoustic      & .348$|$.100          & .352$|$\textbf{.003}          & \textbf{.353}$|$.023    \\
			sensor        & .680$|$.100          & .667$|$\textbf{.009}          & \textbf{.748}$|$.075    \\
			hyperPlane    & .416$|$.100          & .385$|$\textbf{.028}          & \textbf{.417}$|$.100    \\
			electricity   & .295$|$.100          & .266$|$\textbf{.011}          & \textbf{.530}$|$.093    \\
			powersupply   & .650$|$.100          & .630$|$\textbf{.020}          & \textbf{.653}$|$.099    \\
			airlines      & \textbf{.682}$|$.100 & .654$|$\textbf{.027}          & .641$|$.100             \\
			sea           & \textbf{.502}$|$.100 & .489$|$\textbf{.021}          & \textbf{.502}$|$.100             \\
			covtype       & .956$|$.100          & \textbf{.980}$|$\textbf{.001} & .979$|$.001             \\ \bottomrule
		\end{tabular}
	\end{center}
	\caption{Summary of the online performance (left) and model size (right) on the full benchmark suite of the best three algorithms run with budget $10\%$ of the total stream length (model size is also expressed as a fraction of the stream length).}
	\label{tab:data_perf1}
\end{table}
\begin{table}[h]
	\begin{center}
		\begin{tabular}{@{}lccc@{}}
			\toprule
			\textbf{Data} & \textbf{NN3} & \textbf{AUTO-ADJ} & \textbf{AUTO-ADJ FIX} \\ \midrule
			kddcup99      & .550$|$.010    & .501$|$\textbf{.001}         & \textbf{.654}$|$.009    \\
			poker         & .674$|$.010    & .691$|$\textbf{.001}         & \textbf{.710}$|$.010    \\
			connect-4     & .575$|$.010    & .590$|$\textbf{.003}         & \textbf{.603}$|$.010    \\
			acoustic      & .345$|$.010    & .347$|$\textbf{.001}         & \textbf{.349}$|$.009    \\
			sensor        & .614$|$.010    & .620$|$\textbf{.001}         & \textbf{.759}$|$.009    \\
			hyperPlane    & .391$|$.010    & .361$|$\textbf{.003}         & \textbf{.427}$|$.010    \\
			electricity   & .130$|$.010    & .120$|$\textbf{.001}         & \textbf{.621}$|$.010    \\
			powersupply   & .609$|$.010    & .586$|$\textbf{.001}         & \textbf{.622}$|$.009    \\
			airlines      & .634$|$.010    & .590$|$\textbf{.003}         & \textbf{.668}$|$.010    \\
			sea           & .456$|$.010    & .462$|$\textbf{.002}         & \textbf{.473}$|$.009    \\
			covtype       & .945$|$.010    & .975$|$\textbf{.001}         & \textbf{.979}$|$\textbf{.001}    \\ \bottomrule
		\end{tabular}
	\end{center}
	\caption{Summary of the online performance (left) and model size (right) on the full benchmark suite of the three best algorithms run with budget $1\%$ of the total stream length (model size is also expressed as a fraction of the stream length).}
	\label{tab:data_perf2}
\end{table}
As we can observe from these tables, AUTO-ADJ FIX generally outperforms the other methods at the same model sizes. This is very evident on the datasets with drift, such as electricity, and when the budget limit is very small ($1\%$ of the total stream length). Along the same lines of Figure~\ref{fig:budget_summary}, we show in Figure~\ref{fig:budget_summaryFIX} the overall performance of the compared methods using all the budget/rate values $\{1\%,3\%,5\%,10\%\}$.
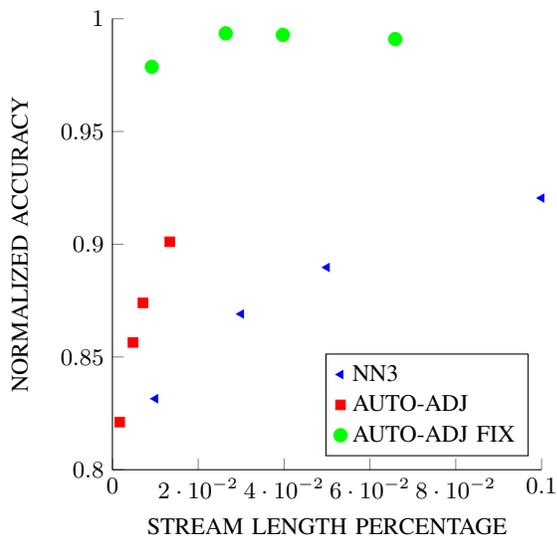
\begin{figure}[h,t]
	\centering
% Minimal pgfplots version: 1.3
%  http://www.mathworks.com/matlabcentral/fileexchange/22022-matlab2tikz
\begin{tikzpicture}

\begin{axis}[%
width=0.95092\figurewidthb,
height=\figureheightb,
at={(0\figurewidthb,0\figureheightb)},
scale only axis,
every outer x axis line/.append style={black},
every x tick label/.append style={font=\color{black}},
xmin=0,
xmax=0.1,
xlabel={STREAM LENGTH PERCENTAGE},
every outer y axis line/.append style={black},
every y tick label/.append style={font=\color{black}},
ymin=0.8,
ymax=1,
ylabel={NORMALIZED ACCURACY},
axis x line*=bottom,
axis y line*=left,
legend style={at={(0.97,0.03)},anchor=south east,legend cell align=left,align=left,draw=black}
]
\addplot [color=blue,mark size=1.7pt,only marks,mark=triangle*,mark options={solid,rotate=90,fill=blue}]
  table[row sep=crcr]{%
0.1	0.92045957384044\\
};
\addlegendentry{NN3};

\addplot [color=red,mark size=1.8pt,only marks,mark=square*,mark options={solid,fill=red}]
  table[row sep=crcr]{%
0.0133699375001397	0.901084182597802\\
};
\addlegendentry{AUTO-ADJ};

\addplot [color=green,mark size=2.5pt,only marks,mark=*,mark options={solid,fill=green}]
  table[row sep=crcr]{%
0.0659001015752364	0.990993400772385\\
};
\addlegendentry{AUTO-ADJ FIX};

\addplot [color=blue,mark size=1.7pt,only marks,mark=triangle*,mark options={solid,rotate=90,fill=blue},forget plot]
  table[row sep=crcr]{%
0.05	0.889722188989033\\
};
\addplot [color=red,mark size=1.8pt,only marks,mark=square*,mark options={solid,fill=red},forget plot]
  table[row sep=crcr]{%
0.00711201451614163	0.874008906340941\\
};
\addplot [color=green,mark size=2.5pt,only marks,mark=*,mark options={solid,fill=green},forget plot]
  table[row sep=crcr]{%
0.0397211740873381	0.992834151274051\\
};
\addplot [color=blue,mark size=1.7pt,only marks,mark=triangle*,mark options={solid,rotate=90,fill=blue},forget plot]
  table[row sep=crcr]{%
0.03	0.869065546987612\\
};
\addplot [color=red,mark size=1.8pt,only marks,mark=square*,mark options={solid,fill=red},forget plot]
  table[row sep=crcr]{%
0.00477804532125264	0.85643074907236\\
};
\addplot [color=green,mark size=2.5pt,only marks,mark=*,mark options={solid,fill=green},forget plot]
  table[row sep=crcr]{%
0.0263870408158675	0.993521017339718\\
};
\addplot [color=blue,mark size=1.7pt,only marks,mark=triangle*,mark options={solid,rotate=90,fill=blue},forget plot]
  table[row sep=crcr]{%
0.01	0.831498747750378\\
};
\addplot [color=red,mark size=1.8pt,only marks,mark=square*,mark options={solid,fill=red},forget plot]
  table[row sep=crcr]{%
0.00169876944075864	0.821117617400158\\
};
\addplot [color=green,mark size=2.5pt,only marks,mark=*,mark options={solid,fill=green},forget plot]
  table[row sep=crcr]{%
0.00918065508045772	0.978680996141109\\
};
\end{axis}
\end{tikzpicture}%
	\caption{
		Online performance against model size, averaged over the datasets. The model size is relative to the stream length, whereas the online performance is measured relative to the top-performing method on each dataset without restriction on model size. 
		\label{fig:budget_summaryFIX}
	}
\end{figure}
AUTO-ADJ FIX clearly outperforms all the other methods. This is not surprising, as AUTO-ADJ FIX has a better way of choosing the data points supporting the model as opposed to the random selection imposed on the other methods.

\section{Conclusion and Future Works}
\label{s:concl}

We presented an intuitive and easy to implement approach for nonparametric classification of data streams. Our more sophisticated algorithms feature the most appealing traits in stream mining applications: nonparametric classification, incremental learning, dynamic addition of new classes, small model size, fast prediction at testing time (logarithmic in the model size), essentially no parameters to tune. We empirically showed the effectiveness of our approach in different scenarios and against several standard baselines. In addition, we proved strong theoretical guarantees on the online performance of the most basic version of our approach.

Further research will focus on finding a confidence measure for the prediction scores, which could be used in a semi-supervised framework (e.g., active learning). Another interesting line of research is concerned with finding a more sophisticated and theoretically justified strategy to keep the model size bounded. A further, very challenging research line is in the direction of taming the curse of dimensionality problem that affects all nonparametric approaches. For instance, we plan on investigating notions of local dimensions that allow to perform dimensionality reduction locally and incrementally. 

%\bibliographystyle{abbrv}
%\bibliography{refs}  % sigproc.bib is the name of the Bibliography in this case

%\appendices

\end{document}